\documentclass{article}[12pt]
\usepackage{rotating,amsmath, amssymb, amsthm,algorithm}
\usepackage{amsfonts, verbatim,url}
\usepackage{epsfig,url, psfrag,appendix}
\usepackage{graphicx, multirow,graphics}
\usepackage{setspace}
\usepackage[left=3.1cm,top=2.5cm,right=3.1cm,bottom=2.5cm]{geometry}
\usepackage{natbib}
\usepackage{color}

\newtheorem{thm}{Theorem} 
\newtheorem{lem}{Lemma}

\newtheorem{rem}{Remark}

\newcommand{\sgn}{\operatorname{sgn}}
\newcommand{\s}{\mathbf S}

\newcommand{\half}{\mbox{$\frac12$}}

\newcommand{\tr}{\,\mathrm{tr}}
\newcommand{\sbt}{\mathrm{subject\;\;to}}

\newcommand{\B}{\boldsymbol}
\newcommand{\M}{\mathbf}
\newcommand{\commentout}[1]{}

\newcommand{\abs}{\mbox{abs}}

\DeclareMathOperator*{\mini}{minimize}
\DeclareMathOperator*{\maxi}{maximize}
\newcommand{\GL}{{\sc glasso}}
\newcommand{\PGL}{{\sc p-glasso}}
\newcommand{\DPGL}{{\sc dp-glasso}}
\newcommand{\DDGL}{{\sc dd-glasso}}

\begin{document}


\title{The Graphical Lasso: New Insights and Alternatives}
\author{Rahul Mazumder\thanks{email: rahulm@stanford.edu}  \hspace{2cm} Trevor Hastie\thanks{email: hastie@stanford.edu}\\
Department of Statistics\\
Stanford University\\
Stanford, CA  94305.}
\date{Revised Draft on August 1, 2012}

\maketitle

\begin{abstract}
  The graphical lasso \citep{FHT2007a} is an algorithm 
for learning the structure in an undirected Gaussian graphical
  model, using $\ell_1$ regularization to control the number of zeros
  in the precision matrix ${\B\Theta}={\B\Sigma}^{-1}$
  \citep{BGA2008,yuan_lin_07}. The {\texttt R}
  package \GL\ \citep{FHT2007a} is popular, fast, and allows one to efficiently build a
  path of models for different values of the tuning parameter.
  Convergence of \GL\ can be tricky; the converged precision matrix
  might not be the inverse of the estimated covariance, and
  occasionally it fails to converge with warm starts. In this paper we
  explain this behavior, and propose new algorithms that appear to
  outperform \GL.

  By studying the ``normal equations'' we see that, \GL\ is solving the {\em dual} of the graphical
  lasso penalized likelihood, by block coordinate ascent; a result which can also be found in \cite{BGA2008}.
  In this dual, the target of estimation is $\B\Sigma$, the covariance matrix,
  rather than the precision matrix $\B\Theta$.  We propose  similar
  primal algorithms \PGL\ and \DPGL, that also operate by block-coordinate descent,
  where $\B\Theta$ is the optimization target.  We study all of these
  algorithms, and in particular different approaches to solving their
  coordinate sub-problems.  We conclude that \DPGL\ is
  superior from several points of view.
\end{abstract}

\section{Introduction}
Consider a data matrix $\M{X}_{n \times p}$, a sample of $n$
realizations from a $p$-dimensional Gaussian distribution with zero
mean and positive definite covariance matrix $\B\Sigma$.  The task is
to estimate the unknown $\B\Sigma$ based on the $n$ samples --- a
challenging problem especially when $n \ll p$, when the
ordinary maximum likelihood estimate does not exist.  Even if it
does exist (for $p \leq n$), the MLE is often poorly behaved, and
regularization is called for.  The Graphical Lasso
\citep{FHT2007a} is a regularization
framework for estimating the covariance matrix $\B\Sigma$, under the
assumption that its inverse $\B\Theta=\B\Sigma^{-1}$ is
sparse~\citep{BGA2008,yuan_lin_07,MB2006}. $\B\Theta$ is called the precision matrix; if an element
$\theta_{jk}=0$, this implies that the corresponding variables $X_j$
and $X_k$ are conditionally independent, given the rest.  Our
algorithms focus either on the restricted version of $\B\Theta$ or its
inverse ${\M W}={\B\Theta}^{-1}$.  The graphical lasso problem minimizes a
$\ell_1$-regularized negative log-likelihood:
\begin{equation} \label{eqn-1}
\mini_{\B\Theta \succ \mathbf{0} } f(\B\Theta) := - \log\det (\B{\Theta} ) + \tr(\s\B{\Theta} ) + \lambda \|\B{\Theta} \|_1.
\end{equation} 
Here $\M{S}$ is the sample covariance matrix, $\|\B{\Theta} \|_1$ denotes the sum of the absolute values of  
$\B{\Theta}$, and $\lambda$ is a tuning parameter controlling the amount of $\ell_1$ shrinkage.
This is a semidefinite programming problem (SDP) in the variable $\B{\Theta}$ \citep{BV2004}.

In this paper we revisit the \GL\ algorithm proposed by
\citet{FHT2007a} for solving~(\ref{eqn-1}); we analyze its properties, expose problems and issues,
and propose alternative algorithms more suitable for the task.

Some of the results and conclusions of this paper can be found in \cite{BGA2008}, both explicitly and implicitly. We
re-derive some of the results and derive new results, insights and algorithms, using a unified and more elementary framework.    
\paragraph{Notation}
We denote the entries of a matrix $\B{A}_{n \times n}$ by $a_{ij}$. 
$\|\B{A}\|_1$ denotes the sum of its absolute values, $\|\B{A}\|_\infty$ the maximum absolute value of its entries, 
$\|\B{A}\|_F$ is its Frobenius norm, and $\abs(\B{A})$ is the matrix with elements $|a_{ij}|$. For a vector $\M{u} \in \Re^{q}$,  $\|\M{u}\|_1$ denotes the $\ell_1$ norm, and so on.

From now on, unless otherwise specified, we will assume that $\lambda >0$.
\section{Review of the \GL\  algorithm.}
\label{sec:review}
We use the frame-work of ``normal equations'' as in
\cite{FHT-09,FHT2007a}. 
Using sub-gradient notation, we can write the optimality conditions (aka ``normal equations") for a solution to (\ref{eqn-1}) as
\begin{equation}\label{stat-1}
- \B{\Theta}^{-1} + \M{S} + \lambda \B\Gamma = \M{0},  
\end{equation}
where $\B\Gamma$ is a matrix of component-wise signs of $\B\Theta$: 
\begin{equation}
  \label{eq:Gamma}
    \begin{array}{rll}
      \gamma_{jk}&=&\mbox{sign}(\theta_{jk})\mbox{ if $\theta_{jk}\neq 0$}\\
      \gamma_{jk}&\in&[-1,1]\mbox{ if $\theta_{jk}=0$}
    \end{array}
\end{equation}
(we use the notation $\gamma_{jk}\in\mbox{Sign}(\theta_{jk})$).
Since the global stationary conditions of (\ref{stat-1}) require $\theta_{jj}$ to be positive, this implies that
\begin{equation}\label{glob-stat-diag1}
w_{ii} = s_{ii} + \lambda,\; i = 1, \ldots, p,
\end{equation}
where $\M{W}=\B{\Theta}^{-1}$.

\GL\ uses a block-coordinate method for solving (\ref{stat-1}).
Consider a partitioning of  $\B{\Theta}$ and $\B{\Gamma}$: 
\begin{eqnarray}\label{break-x}
\B{\Theta} = \left(
  \begin{array}{cc}
    \B{\Theta}_{11} & \B{\theta}_{12} \\
    \B{\theta}_{21}  & \theta_{22} \\
  \end{array}
\right),                 & \B{\Gamma} = \left( 
  \begin{array}{cc}
    \B{\Gamma}_{11} & \B{\gamma}_{12} \\
    \B{\gamma}_{21}  & \gamma_{22} \\
  \end{array}
\right)
\end{eqnarray}
where $\B{\Theta}_{11}$ is $(p-1) \times (p-1)$, $\B{\theta}_{12}$ is $(p-1) \times 1$ and $\theta_{22}$ is scalar. 
$\M W$ and  $\M S$ are partitioned the same way.
Using properties of inverses of block-partitioned matrices, observe that $\M{W}=\B{\Theta}^{-1}$ can be written in two equivalent forms:
\begin{eqnarray} 
\left(
\begin{array}{cc}
    \M{W}_{11} & \M{w}_{12} \\
    \M{w}_{21}  & w_{22} \\
  \end{array}
\right) & = & \left(
  \begin{array}{cc}
(\B{\Theta}_{11} - \frac{\B{\theta}_{12}\B{\theta}_{21}}{\theta_{22}})^{-1} & -\M W_{11}\frac{\B{\theta}_{12}}{\theta_{22}} \\ [10pt]
    \cdot  & \frac{1}{\theta_{22}} - \frac{\B{\theta}_{21} \M W_{11} \B{\theta}_{12}}{\theta^2_{22}} \\
  \end{array}
\right)  \label{partition-1} \\[10pt]
 &= & \left(
  \begin{array}{cc}
\B{\Theta}^{-1}_{11} + \frac{\B\Theta_{11}^{-1}\B{\theta}_{12}\B{\theta}_{21}\B\Theta_{11}^{-1}}{(\theta_{22} - \B{\theta}_{21} \B\Theta_{11}^{-1} \B{\theta}_{12})} & 
-\frac{\B\Theta_{11}^{-1}\B{\theta}_{12}}{\theta_{22} - \B{\theta}_{21}\B\Theta_{11}^{-1}\B{\theta}_{12}}  \\ [10pt]
    \cdot  & \frac{1}{(\theta_{22} - \B{\theta}_{21} \B\Theta_{11}^{-1} \B{\theta}_{12})} \\
  \end{array}
\right). \label{partition-2}
\end{eqnarray}
\GL\ solves for a row/column of (\ref{stat-1}) at a time, holding the rest fixed.
Considering the $p$th column of (\ref{stat-1}), we get
\begin{equation} \label{normal-part1}
- \M{w}_{12}  + \M{s}_{12}  + \lambda \B{\gamma}_{12} =\M 0.
\end{equation}
Reading off $\M{w}_{12}$ from (\ref{partition-1}) we have   
\begin{equation}\label{eq-id-1}
\M{w}_{12} = - \M{W}_{11}\B{\theta}_{12}/\theta_{22} 
\end{equation}
and plugging into (\ref{normal-part1}), we have:
\begin{equation}\label{grad-2}
\M{W}_{11}\frac{\B{\theta}_{12}}{\theta_{22}}+\M{s}_{12} + \lambda \B{\gamma}_{12} = \M{0}.
\end{equation}
\GL\ operates on the above gradient equation, as described below.

As a variation consider reading off $\M{w}_{12}$  from 
(\ref{partition-2}):
\begin{equation}\label{grad-3}
\frac{\B{\Theta}^{-1}_{11}\B{\theta}_{12}}{(\theta_{22} - \B{\theta}_{21} \B\Theta_{11}^{-1} \B{\theta}_{12})}
  + \M{s}_{12} + \lambda \B\gamma_{12} = \M{0}.
\end{equation}
The above simplifies to
\begin{equation}\label{grad-4}
\B{\Theta}^{-1}_{11}\B{\theta}_{12}w_{22} + \M{s}_{12} + \lambda \B{\gamma}_{12} = \M{0},
\end{equation} 
where $w_{22} =1/(\theta_{22} - \B{\theta}_{21} \B{\Theta}^{-1}_{11}\B{\theta}_{12})$ is fixed (by the global stationary conditions (\ref{glob-stat-diag1})).
We will see that these two apparently similar estimating equations  (\ref{grad-2}) and (\ref{grad-4}) lead to \emph{very} different algorithms.

The \GL\ algorithm solves (\ref{grad-2}) for $\B{\beta}=\B{\theta}_{12}/\theta_{22}$, that is
\begin{equation}\label{grad-2a}
\M{W}_{11}\B{\beta}+\M{s}_{12} + \lambda \B{\gamma}_{12} = \M{0},
\end{equation}
where $\B{\gamma}_{12}\in\mbox{Sign}(\B\beta)$, since $\theta_{22}>0$. 
(\ref{grad-2a}) is the stationarity equation 
for the following $\ell_1$ regularized quadratic program: 
\begin{equation} \label{lasso-grad-3}
\mini_{\B\beta\in \Re^{p-1}}\;\;\left\{ \half \B{\beta}'\M{W}_{11}\B{\beta} +  \B{\beta}' \M{s}_{12} + \lambda\|\B{\beta}\|_1 \right\},
\end{equation}
where $\M{W}_{11}\succ 0$ is assumed to be fixed. This is analogous to a
lasso regression problem of the last variable on the rest, except the
cross-product matrix $\M{S}_{11}$ is replaced by its current estimate
$\M{W}_{11}$. This problem itself can be solved efficiently using
elementwise coordinate descent, exploiting the sparsity in
$\B\beta$. From $\hat{\B{\beta}}$, it is easy to obtain
$\hat{\M{w}}_{12}$ from (\ref{eq-id-1}).  Using the lower-right
element of (\ref{partition-1}), $\hat{\theta}_{22}$ is obtained by
\begin{equation}
\frac1{\hat{\theta}_{22}}  = w_{22}- \hat{\B\beta}'\hat{\M w}_{12}.
\label{diag-upd}
\end{equation}
Finally, $\hat{\B{\theta}}_{12}$ can now be recovered from $\hat{\B{\beta}}$ and $\hat{\theta}_{22}$.
Notice, however, that having solved for $\B\beta$ and updated $\M{w}_{12}$, \GL\ can move onto the next block; disentangling $\B{\theta}_{12}$ and $\theta_{22}$ can be done at the end, when the algorithm over all blocks has converged. 
The \GL\ algorithm is outlined in Algorithm~\ref{algo-glasso}. We show in Lemma~\ref{lem:warm-start} in Section~\ref{sec:warm-starts} that the successive updates in \GL\ keep $\M W$ positive definite.
\begin{algorithm}\caption{\GL\ algorithm \citep{FHT2007a}} \label{algo-glasso}
\begin{enumerate}
\item Initialize $\M W = \M S + \lambda \M I$.
\item Cycle around the columns repeatedly, performing the following steps till convergence:
\begin{enumerate}
\item Rearrange the rows/columns so that the target column is last (implicitly).
\item Solve the lasso problem (\ref{lasso-grad-3}), using as warm starts the solution from the previous round for this column. 
\item Update the row/column (off-diagonal) of the covariance  using $\hat{\M w}_{12}$ (\ref{eq-id-1}).
\item Save $\hat{\B\beta}$ for this column in the matrix $\M B$.
\end{enumerate}
\item Finally, for every row/column, compute the diagonal entries $\hat{\theta}_{jj}$ using~(\ref{diag-upd}), and
convert the $\M B$ matrix to  $\B\Theta$.
\end{enumerate}
\end{algorithm}
\begin{figure}[htpb]
  \centering
\includegraphics[width=\textwidth]{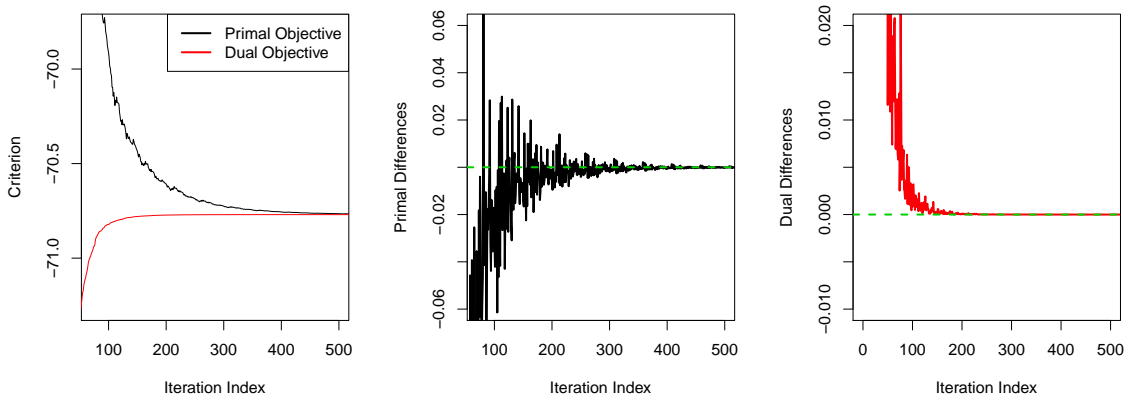}
  \caption{\small\em [Left panel] The objective values of the primal criterion~(\ref{eqn-1}) and the dual criterion~(\ref{box-sdp-1}) corresponding to the covariance matrix $\M{W}$ produced by \GL\ algorithm as a function of the iteration index (each column/row update). [Middle Panel] The successive differences of the primal objective values --- 
the zero crossings indicate non-monotonicity. [Right Panel] The successive differences in the dual objective values --- there are no zero crossings, indicating that \GL\ produces a monotone sequence of dual objective values.}
\label{fig-non-mono1}
\end{figure}

Figure~\ref{fig-non-mono1} (left panel, black curve) plots the objective $f(\B\Theta^{(k)})$ for the sequence of solutions produced by \GL\ on an example. Surprisingly, the curve is not monotone decreasing, as confirmed by the middle plot. If \GL\ were solving (\ref{eqn-1}) by block coordinate-descent, we would not anticipate this behavior. 

A closer look at steps (\ref{eq-id-1}) and (\ref{grad-2}) of the \GL\
algorithm leads to the following observations:
\begin{enumerate}
\item[(a)]
We wish to solve (\ref{normal-part1}) for $\B{\theta}_{12}$. However $\B{\theta}_{12}$ 
 is entangled in $\M{W}_{11}$, which is  (incorrectly) treated as a constant.
\item[(b)] After updating  $\B{\theta}_{12}$, we see from (\ref{partition-2}) that
the entire (working) covariance matrix $\M{W}$ changes. \GL\ however updates
only $\M{w}_{12}$ and $\M{w}_{21}$.
\end{enumerate}
These two observations explain the non-monotone behavior of \GL\ in minimizing $f(\B\Theta)$. 
Section~\ref{primal-cd} shows a corrected block-coordinate descent algorithm  for $\B\Theta$, and Section~\ref{dual-cd}
shows that the \GL\ algorithm is actually optimizing  the dual of problem~(\ref{eqn-1}), with the optimization variable being $\M W$.
\section{A Corrected \GL\   block coordinate-descent algorithm}\label{primal-cd}
Recall that (\ref{grad-4}) is a variant of (\ref{grad-2}), where the
dependence of the covariance sub-matrix $\M{W}_{11}$ on
$\B{\theta}_{12}$ is explicit.  With $\B\alpha=\B{\theta}_{12}w_{22}$
(with $w_{22}\geq 0$ fixed), $\B{\Theta}_{11}\succ 0$, (\ref{grad-4}) is equivalent to the
stationary condition for
\begin{equation}\label{lasso1-primal}
  \mini_{\B\alpha \in \Re^{p-1}} \left\{ \half \B\alpha'\B{\Theta}^{-1}_{11} \B\alpha +  \B\alpha' \M{s}_{12} + \lambda \|\B\alpha\|_1 \right\}.
\end{equation}
If $\hat{\B\alpha}$ is the minimizer of (\ref{lasso1-primal}), then
$\hat{\B{\theta}}_{12}=\hat{\B\alpha}/w_{22}$.  To complete the
optimization for the entire row/column we need to update $\theta_{22}$. This follows simply from (\ref{partition-2})
\begin{equation}
  \label{eq:theta22}
  \hat{\theta}_{22}=\frac1{w_{22}}+\hat{\B\theta}_{21}\B{\Theta}_{11}^{-1}\hat{\B\theta}_{12},
\end{equation}
with $w_{22}=s_{22}+\lambda$.

To solve (\ref{lasso1-primal}) we need
$\B{\Theta}^{-1}_{11}$ for each block update.
We achieve this by maintaining $\M{W}=\B{\Theta}^{-1}$ as the iterations proceed.
Then for each block
\begin{itemize}
\item we obtain $\B{\Theta}_{11}^{-1}$ from
\begin{equation}\label{update-rank-one-1}
\B{\Theta}^{-1}_{11} = \M W_{11} - \M{w}_{12}\M{w}_{21}/w_{22};
\end{equation}
\item once $\B{\theta}_{12}$ is updated, the \emph{entire} working covariance matrix 
$\M W$ is updated (in particular the portions $\M W_{11}$ and $\M{w}_{12}$), via the identities in (\ref{partition-2}), using the known $\B\Theta_{11}^{-1}$.
\end{itemize}
Both these steps are simple rank-one updates with a total cost of $O(p^2)$ operations.

We refer to this as the  primal graphical lasso or \PGL, which we present in Algorithm~\ref{algo-primal-glasso}.  
\begin{algorithm}\caption{\PGL\ Algorithm} \label{algo-primal-glasso}
\begin{enumerate}
\item Initialize $\M{W} = \mathrm{diag}(\s)+\lambda\M{I}$, and $\B\Theta = \M{W}^{-1}$.
\item Cycle around the columns repeatedly, performing the following steps till convergence:
\begin{enumerate}
\item Rearrange the rows/columns so that the target column is last (implicitly).
\item Compute $\B\Theta_{11}^{-1}$ using (\ref{update-rank-one-1}).
\item Solve (\ref{lasso1-primal}) for $\B\alpha$, using as warm starts the solution from the previous round of row/column updates. Update
  $\hat{\B{\theta}}_{12}=\hat{\B\alpha}/w_{22}$, and $\hat{\theta}_{22}$ using (\ref{eq:theta22}).
\item Update $\B\Theta$ and $\M{W}$ using (\ref{partition-2}), ensuring that $\B\Theta\M W= \M I_p$.
\end{enumerate}
\item Output the solution $\B\Theta$ (precision) and its exact inverse $\M W$ (covariance).
\end{enumerate}
\end{algorithm}

The \PGL\ algorithm requires slightly more work than \GL, since an
additional $O(p^2)$ operations have to be performed before and after
each block update. In return we have that after every row/column update, $\B
\Theta$ and $\M W$ are positive definite (for $\lambda >0$) and $\B\Theta\M W= \M I_p$.

\section{What is  \GL\  actually solving?}\label{dual-cd}
Building upon the framework developed in Section~\ref{sec:review}, 
we now proceed to establish that \GL\ solves the convex dual of problem~(\ref{eqn-1}), by block coordinate ascent. 
We reach this conclusion via elementary arguments, closely aligned with the framework we develop in Section~\ref{sec:review}. 
The approach we present here is intended for an audience without much of a familiarity with convex duality theory~\cite{BV2004}.  

Figure~\ref{fig-non-mono1} illustrates that \GL\ is an ascent algorithm on the dual of the problem~\ref{eqn-1}.  
The red curve in the left
plot shows the dual objective rising monotonely, and the rightmost
plot shows that the increments are indeed positive. There is an added twist though: in solving the block-coordinate update, \GL\ solves instead the dual of {\em that} subproblem.

\subsection{Dual of the $\ell_1$ regularized log-likelihood} \label{sec:dual-l1-likhd}
We present below the following lemma, 
the conclusion of which also appears in \cite{BGA2008}, but we use the framework developed in Section~\ref{sec:review}.

\begin{lem}\label{lab:lem-2}
Consider the primal problem (\ref{eqn-1}) and its stationarity conditions (\ref{stat-1}).
These are equivalent to the stationarity conditions for the box-constrained SDP
\begin{equation}\label{box-sdp-1}
\maxi_{\tilde{\B{\Gamma}}:\; \|\tilde{\B{\Gamma}}\|_\infty \leq \lambda}\; g(\tilde{\B{\Gamma}}):=  \log\det(\s + \tilde{\B{\Gamma}}) + p  
\end{equation}
under the transformation $\s + \tilde{\B{\Gamma}} = \B\Theta^{-1}$.
\end{lem}
\begin{proof}
The (sub)gradient conditions (\ref{stat-1}) can be rewritten as:
\begin{equation}\label{stat-11}
 - (\s + \lambda\B\Gamma)^{-1}  + \B\Theta = \M 0 
\end{equation}
where $\B\Gamma = \sgn(\B\Theta)$.  
We write $\tilde{\B{\Gamma}} = \lambda \B\Gamma$ and observe that
$\|\tilde{\B{\Gamma}}\|_\infty \leq \lambda$.
Denote by $\abs(\B\Theta)$ the matrix with element-wise absolute values.

Hence if $(\B\Theta,\B\Gamma)$ satisfy (\ref{stat-11}), the substitutions
\begin{equation}\label{subst-global}
\tilde{\B{\Gamma}} = \lambda\B\Gamma; \;\;  \;\; \M{P} = \abs(\B\Theta)  
\end{equation}
satisfy the following set of equations:
\begin{equation} \label{kkt-dual-gl-cond}
  \begin{array}{rcl}
- (\s + \tilde{\B{\Gamma}})^{-1}  + \M{P} * \sgn(\tilde{\B\Gamma}) &=& \M 0\\
 \M{P}*(\abs(\tilde{\B{\Gamma}}) - \lambda\M{1}_p\M{1}_p') &=&\M 0\\
 \|\tilde{\B{\Gamma}}\|_\infty &\leq &\lambda. 
      \end{array}
\end{equation}
In the above, $\M P$ is a symmetric $p\times p$ matrix with non-negative entries, $\M{1}_p\M{1}_p'$ denotes a $p\times p$ matrix of ones,  and the operator `$*$' denotes element-wise product.
We observe that (\ref{kkt-dual-gl-cond}) are the KKT optimality conditions for the box-constrained SDP~(\ref{box-sdp-1}).
Similarly, the transformations $\B\Theta= \M{P} * \sgn(\tilde{\B\Gamma})$ and  $\B\Gamma=\tilde{\B{\Gamma}}/\lambda$
show that conditions (\ref{kkt-dual-gl-cond}) imply condition~(\ref{stat-11}).
Based on~(\ref{stat-11}) the optimal solutions of the two problems 
(\ref{eqn-1}) and (\ref{box-sdp-1}) are related by $\s + \tilde{\B{\Gamma}} = \B\Theta^{-1}$.
\end{proof}
 
Notice that for the dual, the optimization variable is $\tilde{\B\Gamma}$, with $\s +\tilde{\B\Gamma}={\B\Theta}^{-1}=\M W$. In other words, the dual problem solves for $\M W$ rather than $\B \Theta$, a fact that is suggested  by the \GL\ algorithm.
\begin{rem}
The equivalence of the solutions to problems (\ref{box-sdp-1}) and (\ref{eqn-1}) as described above can also be derived via 
 convex duality theory~\citep{BV2004}, which shows that (\ref{box-sdp-1}) is a dual function of the $\ell_1$ regularized negative log-likelihood  (\ref{eqn-1}). Strong duality holds, hence the optimal solutions of the two problems coincide~\cite{BGA2008}.
\end{rem}

We now consider solving  (\ref{kkt-dual-gl-cond}) for the last block  $\tilde{\B \gamma}_{12}$    (excluding diagonal), holding the rest of $\tilde{\B\Gamma}$ fixed. The corresponding equations are
\begin{equation}  \label{block-dual-11}
  \begin{array}{rcl}
- \B{\theta}_{12}  + \M{p}_{12} * \sgn(\tilde{\B\gamma}_{12}) &=& \M 0\\
 \M{p}_{12}*(\abs(\tilde{\B{\gamma}}_{12}) - \lambda\M{1}_{p-1}) &=&\M 0\\
 \|\tilde{\B{\gamma}}_{12}\|_\infty &\leq &\lambda. 
      \end{array}
\end{equation}
The only non-trivial translation is the ${\B\theta}_{12}$ in the first equation. We must express this in terms of the optimization variable $\tilde{\B\gamma}_{12}$. Since ${\M s}_{12}+\tilde{\B\gamma}_{12}={\M w}_{12}$, using the identities in  (\ref{partition-1}),  we have 
$\M{W}^{-1}_{11} (\M{s}_{12} + \tilde {\B \gamma}_{12})= -\B{\theta}_{12}/\theta_{22}$. 
Since $\theta_{22}>0$, we can redefine $\tilde {\M p}_{12}= {\M p}_{12}/\theta_{22}$, to get
\begin{equation}  \label{block-dual-11a}
  \begin{array}{rcl}
 \M{W}^{-1}_{11} (\M{s}_{12} + \tilde {\B \gamma}_{12}) + \tilde{\M p}_{12} * \sgn(\tilde{\B\gamma}_{12}) &=& \M 0\\
 \tilde{\M{p}}_{12}*(\abs(\tilde{\B{\gamma}}_{12}) - \lambda\M{1}_{p-1}) &=&\M 0\\
 \|\tilde{\B{\gamma}}_{12}\|_\infty &\leq &\lambda. 
      \end{array}
\end{equation}

The following lemma shows that a block update of \GL\ solves~(\ref{block-dual-11a}) (and hence~(\ref{block-dual-11})), a block of stationary conditions for the dual of the graphical lasso problem. Curiously, \GL\ does this not directly, but by solving the dual of the 
QP corresponding to this block of equations. 
\begin{lem}\label{lem:step1}
Assume $\M{W}_{11} \succ \M{0}$. The stationarity equations  
\begin{equation}\label{grad-2-1-0}
\M{W}_{11} \hat{\B\beta}  + \M{s}_{12} + \lambda\hat{\B{\gamma}}_{12}=0,
\end{equation}
where $\hat{\B\gamma}_{12}\in\mbox{Sign}(\hat{\B\beta})$, correspond to the solution  of the $\ell_1$-regularized QP:
\begin{equation}\label{grad-2-1-L1-0}
\mini_{\B\beta\in\Re^{p-1}}\; \half {\B\beta}'\M{W}_{11} {\B\beta} + {\B\beta}'\M{s}_{12} + \lambda\|{\B\beta}\|_1.
\end{equation}
Solving (\ref{grad-2-1-L1-0}) is equivalent to solving the following box-constrained QP:
\begin{equation}\label{box-qp1-0}
\mini_{\B{\gamma}\in \Re^{p-1}} \;\half (\M{s}_{12} + \B{\gamma})' \M W^{-1}_{11} (\M{s}_{12} + \B{\gamma}) \;\;\sbt \;\;\|\B{\gamma}\|_\infty \leq \lambda, 
\end{equation} 
with stationarity conditions given by (\ref{block-dual-11a}),
where the  $\hat{\B\beta}$ and $\tilde{\B{\gamma}}_{12}$ are related by
\begin{equation}\label{prim-dual-corresp-0}
\hat{\B\beta} = - \M W_{11}^{-1}(\M{s}_{12} + \tilde{\B{\gamma}}_{12}).
\end{equation}
\end{lem}
\begin{proof}
(\ref{grad-2-1-0}) is the KKT optimality condition for the $\ell_1$ regularized QP (\ref{grad-2-1-L1-0}).
We rewrite (\ref{grad-2-1-0}) as
\begin{equation}\label{grad-2-2}
\hat{\B\beta}  + \M{W}^{-1}_{11} (\M{s}_{12} + \lambda\hat{\B{\gamma}}_{12})=0. 
\end{equation}
Observe that $\hat{\beta}_i= \sgn(\hat{\beta}_i)|\beta_i|\;\forall i$  and $\|\hat{\B{\gamma}}_{12}\|_\infty \leq 1$.
Suppose $\hat{\B\beta},\;\hat{\B{\gamma}}_{12}$ satisfy (\ref{grad-2-2}), then the substitutions 
\begin{equation}\label{substitute1}
\tilde{\B{\gamma}}_{12}= \lambda\hat{\B{\gamma}}_{12}, \;\;  \;\; \tilde{\M p}_{12}= \abs(\hat{\B\beta})
\end{equation}
in (\ref{grad-2-2}) satisfy  the stationarity conditions (\ref{block-dual-11a}).
It turns out that  (\ref{block-dual-11a}) is equivalent to the KKT optimality conditions of the 
box-constrained QP~(\ref{box-qp1-0}).  
Similarly, we note that if $\tilde{\B{\gamma}}_{12}, \tilde{\M p}_{12}$ satisfy (\ref{block-dual-11a}), then 
the substitution
$$\hat{\B{\gamma}}_{12}= \tilde{\B{\gamma}}_{12}/\lambda ; \;\;  \hat{\B\beta} = \tilde{\M p}_{12}*\sgn(\tilde{\B{\gamma}}_{12})$$  
satisfies (\ref{grad-2-2}). Hence the $\hat{\B\beta}$ and $\tilde{\B{\gamma}}_{12}$ are 
related by~(\ref{prim-dual-corresp-0}).
\end{proof}
 

\begin{rem}\label{lab:rem1}
The above result can also be derived via convex duality theory\citep{BV2004}, where 
(\ref{box-qp1-0}) is actually the Lagrange dual of the $\ell_1$ regularized QP (\ref{grad-2-1-L1-0}), with (\ref{prim-dual-corresp-0}) denoting the primal-dual relationship. 
\cite[Section 3.3]{BGA2008} interpret (\ref{box-qp1-0}) as an $\ell_1$ penalized regression problem (using convex 
duality theory) and explore connections with the set up of~\cite{MB2006}.
\end{rem}

Note that the QP (\ref{box-qp1-0}) is a (partial) optimization over
the variable $\M{w}_{12}$ only (since $\M{s}_{12}$ is fixed); the
sub-matrix $\M{W}_{11}$ remains fixed in the QP. Exactly one
row/column of $\M{W}$ changes when the block-coordinate algorithm of
\GL\ moves to a new row/column, unlike an explicit full matrix update
in $\M W_{11}$, which is required if $\B{\theta}_{12}$ is updated.
This again emphasizes that \GL\ is operating on the covariance matrix
instead of $\B\Theta$.  We thus arrive at the following conclusion:
\begin{thm}\label{lab:lem-3}
\GL\ performs block-coordinate ascent on the box-constrained SDP (\ref{box-sdp-1}),
the Lagrange dual of the primal problem (\ref{eqn-1}). Each of the block steps are themselves box-constrained QPs, which \GL\ optimizes via their Lagrange duals.
\end{thm}
In our annotation perhaps \GL\ should be called \DDGL, since it performs dual block updates for the dual of the graphical lasso problem.
\citet{BGA2008}, the paper that inspired the original \GL\ article \citep{FHT2007a}, also operates on the dual. They however solve the block-updates directly (which are box constrained QPs) using interior-point methods.
\section{A New Algorithm --- \DPGL}
In Section \ref{primal-cd}, we described \PGL, a primal  coordinate-descent method. 
For every row/column we need to solve a lasso problem (\ref{lasso1-primal}), which operates on a quadratic form 
corresponding to the square matrix $\B\Theta_{11}^{-1}$. There are two problems with this approach:
\begin{itemize}
\item the matrix $\B\Theta_{11}^{-1}$ needs to be constructed at every row/column update with complexity $O(p^2)$;
\item $\B\Theta_{11}^{-1}$ is dense.
\end{itemize} 
We now show how a simple modification of the $\ell_1$-regularized QP leads to a box-constrained QP with attractive 
computational properties. 

The KKT optimality conditions for (\ref{lasso1-primal}), following (\ref{grad-4}), can be written as:
\begin{equation}\label{lasso1-primal-kkt1}
\B{\Theta}^{-1}_{11} \B{\alpha} + \M{s}_{12} + \lambda \sgn(\B{\alpha}) = 0.
\end{equation}
Along the same lines of the derivations used in Lemma~\ref{lem:step1}, the condition above is  equivalent to
\begin{equation}\label{dual-primal-cd-kkt}
  \begin{array}{rcl}
\tilde{\M q}_{12}*\sgn(\tilde{\B{\gamma}}) +  \B{\Theta}_{11}(\M{s}_{12} + \tilde{\B{\gamma}})& =& \M{0}\\
\tilde{\M q}_{12}*(\abs(\tilde{\B{\gamma}}) - \lambda \M{1}_{p-1} )&=&0\\
  \|\tilde{\B{\gamma}}\|_\infty &\leq& \lambda  
\end{array}
\end{equation}
for some vector (with non-negative entries) $\tilde{\M q}_{12}$. (\ref{dual-primal-cd-kkt}) are the KKT optimality conditions for the following box-constrained QP:
\begin{equation}\label{dual-primal-cd}
\mini_{{\B\gamma} \in \Re^{p-1}} \; \half (\M{s}_{12} + {\B\gamma})'\B{\Theta}_{11}(\M{s}_{12} 
+ {\B\gamma}); \;\;  
\sbt \;\; \|{\B\gamma}\|_\infty \leq \lambda .
\end{equation}
The optimal solutions of (\ref{dual-primal-cd}) and (\ref{lasso1-primal-kkt1}) are related by
\begin{equation}
  \label{eq:dpgl}
\hat{\B{\alpha}} =   -\B{\Theta}_{11}(\M{s}_{12} + \tilde{\B{\gamma}}),
\end{equation}
 a consequence of~(\ref{lasso1-primal-kkt1}), with $\hat{\B{\alpha}}=\hat{\B\theta}_{12}\cdot w_{22}$ and $w_{22}=s_{22}+\lambda$.
The diagonal $\theta_{22}$ of the precision matrix  is updated via (\ref{partition-2}):
\begin{equation}\label{update-theta22}
\hat\theta_{22}= \frac{1 - (\M{s}_{12} + \tilde{\B{\gamma}})'\hat{\B{\theta}}_{12}}{w_{22}}
\end{equation}

By strong duality, the box-constrained QP (\ref{dual-primal-cd}) with
its optimality conditions (\ref{dual-primal-cd-kkt}) is equivalent to
the lasso problem (\ref{lasso1-primal}).
Now both the problems listed at the beginning of the section are removed.
The problem matrix ${\B\Theta}_{11}$ is sparse, and no $O(p^2)$ updating is required after each block.
 
\begin{algorithm}\caption{\DPGL\ algorithm} \label{algo-glasso-sparse}
\begin{enumerate}
\item Initialize $\B{\Theta} = \mathrm{diag}(\s+\lambda \M I)^{-1}$.
\item Cycle around the columns repeatedly, performing the following steps till convergence:
\begin{enumerate}
\item Rearrange the rows/columns so that the target column is last (implicitly).
\item  Solve (\ref{dual-primal-cd}) for $\tilde{\B\gamma}$ and update  
$$\hat{\B{\theta}}_{12}=- \B\Theta_{11}(\M{s}_{12} + \tilde{\B{\gamma}})/ w_{22}$$
\item Solve for $\theta_{22}$ using (\ref{update-theta22}).
\item Update the working covariance $\M{w}_{12}=\M{s}_{12}  + \tilde{\B{\gamma}}$.
\end{enumerate}
\end{enumerate}
\end{algorithm}

The solutions returned at step 2(b) for $\hat{\B\theta}_{12}$ need not
be exactly sparse, even though it purports to produce the solution to
the primal block problem (\ref{lasso1-primal}), which is sparse. One
needs to use a tight convergence criterion when solving (\ref{dual-primal-cd}).  In addition, one can
threshold those elements of $\hat{\B\theta}_{12}$ for which
$\tilde{\B\gamma}$ is away from the box boundary, since those values are known to be zero.

Note that \DPGL\ does to the primal formulation~(\ref{eqn-1}) what \GL\ does to the dual. 
\DPGL\ operates on the precision matrix, whereas \GL\ operates on the covariance matrix.

\section{Computational Costs in Solving the Block QPs}\label{sec:solve-cd}
The $\ell_1$ regularized QPs appearing in (\ref{lasso-grad-3}) and (\ref{lasso1-primal})
 are of the generic form
\begin{equation}\label{gen-l1-qp}
\mini_{\M{u} \in \Re^{q}} \quad \half \M{u}'\M{A}\M{u} + \M{a}'\M{u} + \lambda \|\M{u}\|_1, 
\end{equation}
for $\M{A} \succ \M{0}$.  In this paper, we choose to use cyclical
coordinate descent for solving~(\ref{gen-l1-qp}), as it is used in
the \GL\ algorithm implementation of \citet{FHT2007a}. Moreover,
cyclical coordinate descent methods perform well with good
warm-starts.
These  are available for both (\ref{lasso-grad-3}) and 
(\ref{lasso1-primal}), since they both maintain working copies of the precision matrix, updated
after every row/column update.  There are other efficient
ways for solving (\ref{gen-l1-qp}), capable of scaling to large
problems --- for example first-order proximal methods \citep{fista-09,nest-07}, but we do not pursue them in this paper.

The box-constrained QPs appearing in (\ref{box-qp1-0}) and (\ref{dual-primal-cd})
 are of the generic form:
\begin{equation}\label{gen-box-qp}
\mini_{\M{v}\in \Re^{q}} \quad \half (\M{v} + \M{b} )'\tilde{\M A}(\M{v} + \M{b})\;\;  \sbt \; \|\M{v}\|_\infty \leq \lambda
\end{equation}
for some $\tilde{\M A} \succ \M{0}$.  As in the case above, we will
use cyclical coordinate-descent for
optimizing~(\ref{gen-box-qp}). 

In general it is more efficient to
solve (\ref{gen-l1-qp}) than (\ref{gen-box-qp}) for larger values of $\lambda$.
 This is because a large
value of $\lambda$ in (\ref{gen-l1-qp}) results in sparse solutions
$\hat{\M{u}}$; the coordinate descent algorithm can easily detect when a zero stays zero, and no further work gets done for that coordinate on that pass. 
 If the solution to (\ref{gen-l1-qp}) has
$\kappa$ non-zeros, then on average $\kappa$ coordinates need to be
updated. This leads to a cost of $O(q\kappa)$, for one full sweep
across all the $q$ coordinates.
  
On the other hand, a large $\lambda$ for (\ref{gen-box-qp}) corresponds to a weakly-regularized
solution. Cyclical coordinate procedures for this task are not as
effective. Every coordinate update of $\M{v}$ results in updating the
gradient, which requires  adding a scalar multiple of a column of $\tilde{\M A}$. 
If $\tilde{\M A}$ is dense, this leads to a cost of
$O(q)$, and for one full cycle across all the coordinates this costs
$O(q^2)$, rather than the $O(q\kappa)$ for (\ref{gen-l1-qp}).

However, our experimental results show that \DPGL\ is more efficient
than \GL, so there are some other factors in play.  When $\tilde{\M
  A}$ is sparse, there are computational savings.  If $\tilde{\M A}$
has $\kappa q$ non-zeros, the cost per column reduces on average to $O(\kappa q)$
from $O(q^2)$.  For the formulation (\ref{dual-primal-cd}) $\tilde{\M
  A}$ is $\B{\Theta}_{11}$, which is sparse for large $\lambda$.
Hence for large $\lambda$, \GL\ and \DPGL\ have similar costs.

For smaller values of $\lambda$, the box-constrained QP
(\ref{gen-box-qp}) is particularly attractive.   
Most of the coordinates in the optimal
solution $\hat{\M{v}}$ will pile up at the boundary points
$\{-\lambda,\lambda\}$, which means that the coordinates need not be
updated frequently. For problem  (\ref{dual-primal-cd}) this number is also $\kappa$, 
the number of non-zero coefficients in  the corresponding column of the precision matrix.
If $\kappa$ of the coordinates pile up at
the boundary, then one full sweep of cyclical coordinate descent
across all the coordinates will require updating gradients
corresponding to the remaining $q-\kappa$ coordinates.
Using similar calculations as before, this will cost $O(q(q-\kappa))$ operations per full cycle 
(since for small $\lambda$, $\tilde{\M A}$ will be dense).
For the $\ell_1$ regularized problem (\ref{gen-l1-qp}), no such saving is achieved, and the cost is $O(q^2)$ per cycle.

Note that to solve problem~(\ref{eqn-1}), we  need to solve a QP of a particular type~(\ref{gen-l1-qp}) or 
(\ref{gen-box-qp}) for a certain number of outer cycles (ie full sweeps across rows/columns). For every row/column update, the associated QP requires varying number of iterations to converge. It is hard to characterize all these factors and come up
with precise estimates of convergence rates of the overall algorithm.
However, we have observed that with warm-starts, on a relatively dense grid of $\lambda$s, the complexities given above are 
pretty much accurate for \DPGL\ (with warmstarts) specially when one is interested in solutions with small / moderate accuracy. 
Our experimental results in Section~\ref{sec:expt-simu} and Appendix Section~\ref{sec:times} support our observation.


We will now have a more critical look at the updates of the \GL\ algorithm and study their properties.

\section{\GL : Positive definiteness, Sparsity and Exact Inversion}\label{pos-def-sparse-exact}
As noted earlier, \GL\  operates on $\M{W}$ --- it does \emph{not} explicitly compute the inverse $\M{W}^{-1}$. It does however keep track of the estimates 
for ${\B{\theta}}_{12}$ after every row/column update. The copy of $\B\Theta$ retained by \GL\ along the row/column updates
is not the exact inverse of the optimization variable $\M{W}$.
Figure~\ref{fig-not-sparse-pd} illustrates this by plotting the squared-norm $\|(\B\Theta -  \M{W}^{-1})\|^2_F$ as a function of the iteration index. Only upon (asymptotic) convergence, will $\B{\Theta}$ be equal to 
$\M{W}^{-1}$. This can have important consequences. 

\begin{figure}[htpb]
  \centering
 \includegraphics[width=.8\textwidth]{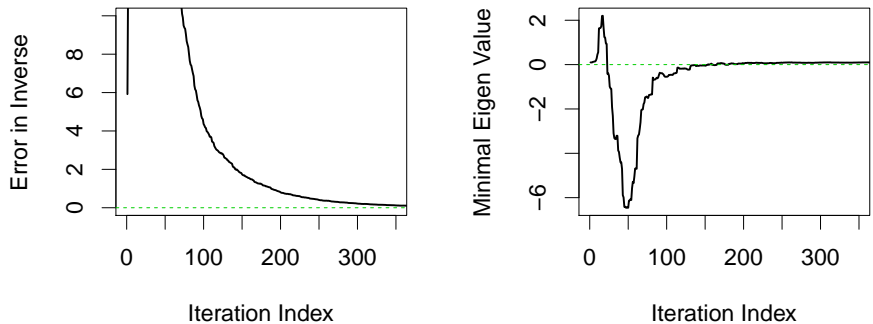}
 \caption{\small\em Figure illustrating some negative properties of \GL\ using a typical numerical example.  [Left Panel] The precision matrix produced after every row/column update need not be the exact inverse of the working covariance matrix --- the squared Frobenius norm of the error is being plotted across iterations. [Right Panel] The estimated precision matrix $\B{\Theta}$ produced by \GL\ need not be positive definite along iterations; plot shows minimal eigen-value.}
\label{fig-not-sparse-pd}
\end{figure}

In many real-life problems one only needs an approximate solution to~(\ref{eqn-1}):
\begin{itemize}
\item for computational reasons it might be impractical to obtain a solution of high accuracy;
\item from a statistical viewpoint it might be sufficient to obtain an approximate solution for $\B\Theta$ that is
\emph{both} sparse and positive definite
\end{itemize}
It turns out that the  \GL\ algorithm is not suited to this purpose.

Since the \GL\ is a block coordinate procedure on the covariance
matrix, it maintains a positive definite covariance matrix at every
row/column update. However, since the estimated precision matrix is
not the exact inverse of $\M{W}$, it need not be positive definite.
Although it is relatively straightforward to maintain an exact inverse
of $\M{W}$ along the row/column updates (via simple rank-one updates
as before), this inverse $\M{W}^{-1}$ need \emph{not} be sparse.
Arbitrary thresholding rules may be used to set some of the entries to
zero, but that might destroy the  positive-definiteness of the
matrix. Since a principal motivation of
solving~(\ref{eqn-1}) is to obtain a sparse precision matrix (which is
also positive definite), returning a dense $\M{W}^{-1}$ to~(\ref{eqn-1}) is not
desirable.
 
Figures~\ref{fig-not-sparse-pd}
 illustrates the above observations on a typical example.

 The \DPGL\ algorithm operates on the primal~(\ref{eqn-1}). Instead of
 optimizing the $\ell_1$ regularized QP (\ref{lasso1-primal}), which
 requires computing $\B{\Theta}^{-1}_{11}$, \DPGL\ optimizes
 (\ref{dual-primal-cd}). After every row/column update the precision
 matrix $\B{\Theta}$ is positive definite. The working covariance
 matrix maintained by \DPGL\
via $\M{w}_{12}:= \M{s}_{12} + \hat{\B{\gamma}}$ need not be the exact inverse of $\B{\Theta}$. Exact covariance matrix estimates,
if required, can be obtained by tracking $\B{\Theta}^{-1}$ via simple rank-one updates, as described earlier. 

Unlike \GL, \DPGL\ (and \PGL) return a sparse and positive definite precision matrix even if the row/column iterations are terminated prematurely.


\section{Warm Starts and Path-seeking Strategies}\label{sec:warm-starts}
Since we seldom know in advance a good value of $\lambda$, we often compute a sequence of solutions to (\ref{eqn-1}) for a (typically) decreasing sequence of values
$\lambda_1>\lambda_2>\ldots>\lambda_K$.
Warm-start or continuation methods use the solution at $\lambda_i$ as an initial guess for the solution at
 $\lambda_{i+1}$, and often yield great efficiency.
It turns out that for algorithms like \GL\ which operate on the dual problem, not all warm-starts necessarily lead to a convergent algorithm. We address this aspect in detail in this section.

The following lemma states the conditions under which the row/column updates of the \GL\ algorithm 
will maintain positive definiteness of the covariance matrix $\M{W}$.  
\begin{lem}\label{lem:warm-start}
Suppose $\M{Z}$ is used as a warm-start for the \GL\ algorithm. 
If $\M{Z} \succ \M{0}$ and $\| \M{Z} - \s \|_\infty \leq \lambda$,
then every row/column update of \GL\ maintains positive definiteness of the working covariance matrix $\M{W}$.
\end{lem}
\begin{proof}
Recall that the \GL\ solves the dual~(\ref{box-sdp-1}). 
Assume $\M Z$ is partitioned as in (\ref{break-x}), and the $p$th row/column is being updated.
Since $\M{Z} \succ \M{0}$,  we have both
\begin{equation}
  \M{Z}_{11} \succ \M{0} \mbox{ and } \left(z_{22} - \M{z}_{21}(\M{Z}_{11})^{-1}\M{z}_{12} \right) > 0 .
\end{equation}
Since $\M{Z}_{11}$ remains fixed, it suffices to show that after the
row/column update, the expression $(\hat{w}_{22} -
\hat{\M{w}}_{21}(\M{Z}_{11})^{-1}\hat{\M{w}}_{12} )$ remains positive.
Recall that, via standard optimality conditions we have $\hat{w}_{22}
= s_{22} + \lambda$, which makes $\hat{w}_{22} \geq z_{22}$ 
(since by assumption, $|z_{22} - s_{22}|\leq \lambda$ and $z_{22} > 0$). 
Furthermore, $\hat{\M{w}}_{21}= \M{s}_{21} + \hat{\B{\gamma}}$, where
$\hat{\B{\gamma}}$ is the optimal solution to the corresponding
box-QP~(\ref{box-qp1-0}). Since the starting solution $\M{z}_{21}$
satisfies the box-constraint~(\ref{box-qp1-0}) i.e.  $\|\M{z}_{21} -
\M{s}_{21}\|_\infty \leq \lambda$, the optimal solution of the
QP~(\ref{box-qp1-0}) improves the objective:
$$ \hat{\M{w}}_{21}(\M{Z}_{11})^{-1}\hat{\M{w}}_{12}  \leq  {\M{z}}_{21}(\M{Z}_{11})^{-1}{\M{z}}_{12}$$
Combining the above along with the fact that $\hat{w}_{22} \geq z_{22}$ we see
\begin{equation}
\hat{w}_{22} - \hat{\M{w}}_{21}(\M{Z}_{11})^{-1}\hat{\M{w}}_{12} > 0,  
\end{equation}
which implies that the new covariance estimate $\widehat{\M W} \succ \M{0}$.
\end{proof}

\begin{rem}\label{rem:num:3}
If the condition  $\| \M{Z} - \s \|_\infty \leq \lambda$  appearing in 
Lemma~\ref{lem:warm-start} is violated, then the row/column update of \GL\ need not maintain PD of the covariance matrix $\M{W}$.
\end{rem}
We have encountered many counter-examples that show this to be true, see the discussion below.

The {\tt R} package implementation of \GL\ allows the user to specify a
warm-start as a tuple $(\B\Theta_0,\M{W}_0)$. This option is typically used
in the construction of  a path algorithm.

If $(\widehat{\B{\Theta}}_{\lambda},\widehat{\M{W}}_{\lambda})$ is
provided as a warm-start for $\lambda' < \lambda$, then the \GL\
algorithm is not guaranteed to converge. It is easy to find numerical
examples by choosing the gap $\lambda-\lambda'$ to be large enough.
Among the various examples we encountered, we briefly describe one here. 
Details of the experiment/data and other examples can be found in the
online Appendix~\ref{appendix:example1}. We generated a data-matrix $\M{X}_{n\times p }$, with $n=2, p=5$
with iid standard Gaussian entries. $\s$ is the sample covariance matrix. We solved problem~(\ref{eqn-1}) using \GL\ for
$\lambda=0.9\times \max_{i \neq j} |s_{ij}|$. We took the estimated covariance
and precision matrices: $\widehat{\M{W}}_\lambda$ and
 $\widehat{\B{\Theta}}_\lambda$ as a warm-start for the \GL\
 algorithm with $\lambda'=\lambda\times 0.01$. The \GL\ algorithm
 failed to converge with this warm-start. 
We note that  $\|\widehat{\M{W}}_{\lambda} - \s \|_\infty =0.0402 \nleq \lambda'$
(hence violating the sufficient condition in Lemma~\ref{lem:warm-start-pgl}) and
after updating the first row/column via the \GL\ algorithm we observed that 
``covariance matrix'' $\M W$ has negative eigen-values ---
leading to a non-convergent algorithm.    
The above phenomenon is not surprising and easy to explain and generalize.
Since $\widehat{\M{W}}_{\lambda}$ solves
the dual (\ref{box-sdp-1}), it is necessarily of the form
$\widehat{\M{W}}_{\lambda} = \s + {\tilde{\B{\Gamma}}}$, for
$\|{\tilde{\B{\Gamma}}}\|_\infty \leq \lambda$.  In the light of
Lemma~\ref{lem:warm-start} and also Remark \ref{rem:num:3}, the
warm-start needs to be dual-feasible in order to guarantee that the
iterates $\widehat{\M W}$ remain PD and hence for the sub-problems to be
well defined convex programs.  Clearly $\widehat{\M{W}}_{\lambda}$ does not satisfy
the box-constraint $\|\widehat{\M{W}}_{\lambda} - \s \|_\infty \leq
\lambda'$, for $\lambda' < \lambda$.  
However, in practice the  \GL\ algorithm is usually seen to converge (numerically) when $\lambda'$ is quite \emph{close} to $\lambda$. 

The following lemma establishes that any PD matrix can be taken as a warm-start for \PGL\ or \DPGL to ensure a convergent algorithm. 
\begin{lem}\label{lem:warm-start-pgl}
Suppose ${\B\Phi} \succ \M{0}$ is a used as a warm-start for the \PGL\ (or \DPGL) algorithm.
Then every row/column update of \PGL\ (or \DPGL)  maintains positive definiteness of the working precision matrix $\B{\Theta}$.
\end{lem}
\begin{proof}
Consider updating the $p$th row/column of the precision matrix.
The condition ${\B\Phi} \succ \M{0}$ is equivalent to both
$$ {\B\Phi}_{11} \succ \M{0}  \mbox{ and } \left({\phi}_{22} - {\B\Phi}_{21}({\B\Phi}_{11})^{-1}{\B\Phi}_{12}\right)>0 .$$
Note that the block ${\B\Phi}_{11}$ remains fixed; only the $p$th row/column of $\B\Theta$ changes.
${\B\phi}_{21}$ gets updated to $\hat{\B{\theta}}_{21}$, as does $\hat{\B{\theta}}_{12}$. From (\ref{partition-2}) the updated 
diagonal entry $\hat{\theta}_{22}$ satisfies:
$$\hat{\theta}_{22} -  \hat{\B{\theta}}_{21}({\B\Phi}_{11})^{-1}
\hat{\B{\theta}}_{12}  = \frac{1}{(s_{22} + \lambda)} > 0 .$$ 
Thus the updated matrix $\hat{\B{\Theta}}$ remains PD.
The result for the \DPGL\ algorithm follows, since both the versions \PGL\ and \DPGL\ solve the same 
block coordinate problem.
\end{proof}

\begin{rem}\label{rem:num:4}
A simple consequence of Lemmas~\ref{lem:warm-start} and \ref{lem:warm-start-pgl} is that the QPs arising in the process, namely 
the $\ell_1$ regularized QPs (\ref{lasso-grad-3}), (\ref{lasso1-primal})  and the box-constrained QPs (\ref{box-qp1-0}) and (\ref{dual-primal-cd}) are all valid convex programs, since all the respective matrices $\M{W}_{11}$,  $\B{\Theta}^{-1}_{11}$ and $\M{W}^{-1}_{11}$, 
$\B{\Theta}_{11}$ appearing in the quadratic forms are PD.
\end{rem}

As exhibited in Lemma~\ref{lem:warm-start-pgl}, both the algorithms \DPGL\ and \PGL\  are guaranteed to converge from any positive-definite warm start. This is due to the unconstrained formulation of the primal problem~(\ref{eqn-1}).

\GL\ really only requires an initialization for $\M W$, since it
constructs $\B\Theta$ on the fly.  Likewise \DPGL\ only requires an
initialization for $\B\Theta$. Having the other half of the tuple
assists in the block-updating algorithms. For example, \GL\ solves a
series of lasso problems, where $\B\Theta$ play the role as
parameters. By supplying $\B\Theta$ along with $\M W$, the block-wise
lasso problems can be given starting values close to the solutions.
The same applies to \DPGL. In neither case do the pairs have to be
inverses of each other to serve this purpose.

If we wish to start with inverse pairs, and maintain such a
relationship, we have described earlier how $O(p^2)$ updates after
each block optimization can achieve this. One caveat for \GL\ is that starting with an
inverse pair costs $O(p^3)$ operations, since we typically start with $\M W=\s+\lambda\M I$. 
For \DPGL, we typically start with a diagonal matrix, which is trivial to invert.

\section{Experimental Results \& Timing  Comparisons}\label{rev:sec:timings}
We compared the performances of algorithms \GL\ and \DPGL\
(both with and without warm-starts) on
different examples with varying $(n,p)$ values. While
most of the results are presented in this section, some are relegated
to the online Appendix~\ref{sec:times}. Section~\ref{sec:expt-simu}
describes some synthetic examples and
Section~\ref{sec:expt-real} presents comparisons on a
real-life micro-array data-set. 
\subsection{Synthetic Experiments}\label{sec:expt-simu}
In this section we present examples generated from two different covariance
models --- as characterized by the covariance matrix $\B\Sigma$ or
equivalently the precision matrix  $\B\Theta$. 
We create a data matrix $\M{X}_{n \times p}$ by drawing $n$
independent samples from a $p$ dimensional normal
distribution $\mathrm{MVN}(\M{0}, \B\Sigma)$. The sample
covariance matrix is taken as the input $\s$ to problem~(\ref{eqn-1}).
The two covariance models are described below:
\begin{description}
\item[Type-1]  The population
concentration matrix $\B\Theta= {\B\Sigma}^{-1}$ has
uniform sparsity with approximately $77$ \% of the entries zero. 

We created the covariance matrix as follows. We generated a 
matrix $\M{B}$ with iid standard Gaussian entries, symmetrized it
via $\frac{1}{2}(\M{B} + \M{B}')$ and set approximately $77$\% of the entries of this
matrix to zero, to obtain $\tilde{\M{B}}$ (say). 
We added a scalar multiple of the $p$ dimensional identity matrix to 
$\tilde{\M{B}}$ to get the precision matrix $\B{\Theta}=\tilde{\M{B}} + \eta
\M{I}_{p \times p}$, with $\eta$ chosen such that the minimum eigen value
of $\B\Theta$ is one.

\item[Type-2] This example, taken from \cite{yuan_lin_07},
is an auto-regressive process of order two ---
 the precision matrix being tri-diagonal:
\[ \theta_{ij} = \begin{cases} 0.5, & \mbox{if }\;\; |j-i|=1,\; i=2,\ldots,
  (p-1) ; \\ 
0.25, & \mbox{if } \;\;  |j-i|=2,\; i=3,\ldots,(p-2);\;\;  \\
1, & \mbox{if }\;\; i=j, \; i = 1, \ldots, p; \mbox{and} \;\;\\
0 &  \mbox{otherwise}. \end{cases}\]
\end{description}

For each of the two set-ups Type-1 and Type-2 we consider twelve
different combinations of $(n,p)$:
\begin{itemize}
\item[(a)]  $p=1000$,  $n \in \{1500,1000,500\}$.
\item[(b)]  $p=800$,   $n \in \{1000,800,500\}$.
\item[(c)]  $p=500$,   $n \in \{800,500,200\}$.
\item[(d)]  $p=200$,   $n \in \{500,200,50\}$.
\end{itemize}

For every $(n,p)$ we solved~(\ref{eqn-1}) on a
grid of twenty $\lambda$ values linearly spaced in the log-scale, with 
$\lambda_i = 0.8^{i}\times\{ 0.9\lambda_{\max}\}, \; i = 1,\ldots, 20 $, where 
$\lambda_{\max}=\max_{i\neq j} |s_{ij}|$, is the off-diagonal
entry of $\s$ with largest absolute value. $\lambda_{\max}$ is  the smallest value of $\lambda$ for which the
solution to (\ref{eqn-1}) is a diagonal matrix. 

Since this article focuses on the \GL\ algorithm, its properties and
alternatives that stem from the main idea of block-coordinate
optimization, we present here the performances of the following algorithms:
\begin{description}
\item[Dual-Cold] \GL\  with initialization $\M{W}=\s + \lambda
  \M{I}_{p \times p}$, as suggested in \cite{FHT2007a}.
\item[Dual-Warm] The path-wise version of \GL\ with warm-starts, as
  suggested in  \cite{FHT2007a}.  Although this path-wise version need not converge in general, this was not a problem  in our experiments, probably due to the fine-grid of $\lambda$ values.
\item[Primal-Cold] \DPGL\  with diagonal initialization $\B\Theta =(\mbox{diag}(\s)+\lambda\M{I})^{-1}.$
\item[Primal-Warm] The path-wise version of \DPGL\ with warm-starts.
\end{description}
We did not include \PGL\ in the comparisons above since 
\PGL\  requires  additional matrix rank-one updates after every row/column 
update, which makes it more expensive. None of the above listed algorithms require matrix inversions (via rank one updates). 
Furthermore, \DPGL\ and \PGL\  are quite similar as both are doing a block coordinate optimization on the dual. Hence 
we only included \DPGL\ in our comparisons. 
 We used our own implementation of the
\GL\ and \DPGL\ algorithm in R.  The entire program is written in R,
except the inner block-update solvers, which are the real work-horses:
\begin{itemize}
\item For \GL\ we used the lasso code {\texttt crossProdLasso} written in FORTRAN by \cite{FHT2007a};
\item For \DPGL\ we wrote our own FORTRAN code to solve the box QP.
\end{itemize}
An R package implementing \DPGL\ will be made available in CRAN.

In the figure and tables that follow below, for every algorithm, at a fixed $\lambda$ we report the \emph{total time} 
taken by \emph{all} the QPs --- the $\ell_1$ regularized QP for \GL\ and the box constrained QP for \DPGL\ till convergence
All computations were done on a Linux machine with model specs:
{\texttt Intel(R) Xeon(R) CPU 5160  @ 3.00GHz.}

\bigskip

\noindent\textbf{Convergence Criterion:} Since \DPGL\ operates on the the primal formulation and \GL\ operates
on the dual --- to make the convergence criteria comparable across
examples we based it on the relative change in the primal objective values
i.e. $f(\B\Theta)$ (\ref{eqn-1}) across two successive
iterations:
\begin{equation}\label{conv-crit1}
\frac{f(\B\Theta_{k}) - f(\B\Theta_{k-1})}{| f(\B\Theta_{k-1})| } \leq \mathrm{TOL},
\end{equation}
where one iteration refers to a full sweep across
$p$ rows/columns of the precision matrix (for \DPGL\ ) and covariance
matrix (for \GL\ ); and TOL denotes the tolerance level or level of accuracy of the solution. 
To compute the primal objective value for the \GL\ algorithm, the precision matrix is
computed from $\widehat{\M{W}}$ via direct inversion (the time taken for inversion and objective value computation is not included in the timing comparisons).
 
Computing the objective function is quite expensive relative to the computational cost of the iterations. In our experience 
convergence criteria based on a relative change in the precision matrix for \DPGL\  and the covariance matrix for \GL\ seemed to be a  practical choice for the examples we considered. However, for reasons we described above, we used criterion~\ref{conv-crit1} in the experiments.  

\bigskip

\noindent\textbf{Observations:}
Figure~\ref{fig-times1} presents the times taken by the
algorithms to converge to an accuracy of $\mathrm{TOL}=10^{-4}$ on a grid of $\lambda$ values. 
 
The figure shows eight different scenarios with $p > n$, corresponding to the two different 
covariance models Type-1 (left panel) and Type-2 (right panel). It is
quite evident 
that \DPGL\  with warm-starts (Primal-Warm) outperforms all the other
algorithms across all the different examples. All the algorithms
converge quickly for large values of $\lambda$ (typically high
sparsity) and become slower with decreasing $\lambda$. For large $p$
and small $\lambda$, convergence is slow; however for $p> n$, the non-sparse end of the 
regularization path is really not that interesting from a statistical viewpoint.
Warm-starts apparently do \emph{not} always help in speeding up the convergence of  \GL\ ;  
for example see Figure~\ref{fig-times1} with $(n,p)=(500,1000)$ (Type 1) and $(n,p)=(500,800)$ (Type 2). This probably further validates the 
fact that warm-starts in the case of \GL\ need to be carefully designed, in order for them to \emph{speed-up} convergence. 
Note however, that \GL\ with the warm-starts prescribed is not even guaranteed to converge --- we however did not come across any such instance among the experiments presented in this section.

Based on the suggestion of a referee we annotated the plots in Figure~\ref{fig-times1} with locations in the regularization path 
that are of interest. 
For each plot, two vertical dotted lines are drawn which
correspond to the $\lambda$s at which the distance of the estimated precision matrix $\widehat{\B{\Theta}}_\lambda$ 
from the 
population precision matrix is minimized wrt to the $\|\cdot\|_1$ norm (green) and $\|\cdot\|_F$ norm (blue). 
The optimal $\lambda$ corresponding to the $\|\cdot\|_1$ metric chooses sparser models than those chosen by 
$\|\cdot\|_F$; the performance gains achieved by \DPGL\ seem to be more prominent for the latter $\lambda$.

Table~\ref{tab:type1} presents the timings for all the four algorithmic variants on the twelve different $(n,p)$ combinations listed above for Type 1. For every example,
we report the total time till convergence on a grid of twenty $\lambda$ values 
for two different tolerance levels: $\mathrm{TOL}
\in \{10^{-4},10^{-5}\}$. 
Note that the \DPGL\ returns positive definite and sparse precision matrices even if the algorithm is terminated at a relatively small/moderate accuracy level --- this is not the case in \GL\ . 
 The rightmost column presents the 
proportion of non-zeros averaged across the entire path of solutions
$\widehat{\B\Theta}_{\lambda}$, where $\widehat{\B\Theta}_{\lambda}$
is obtained by solving~(\ref{eqn-1}) to a high precision i.e. $10^{-6}$, by
algorithms \GL\ and \DPGL\  and averaging the results.

Again we see that in all the examples \DPGL\ with warm-starts is the clear winner among its competitors. 
For a fixed $p$, the total time to trace out the path generally decreases with increasing $n$. There is no 
clear winner between \GL\ with warm-starts and \GL\ without
warm-starts. It is often seen that \DPGL\ without warm-starts
converges faster than both the variants of \GL\ (with and without
warm-starts).

Table~\ref{tab:type2} reports the timing comparisons for
Type~2. Once again we see that in all the examples Primal-Warm turns
out to be the clear winner.  

For $n \leq p=1000$, we observe that Primal-Warm is generally
faster for Type-2 than Type-1. This however, is reversed for
smaller values of $p \in \{800,500\}$.  Primal-Cold is has a smaller
overall computation time for Type-1 over Type-2.
In some cases (for example  $n \leq p=1000$), we see that
Primal-Warm in Type-2 converges much faster than its competitors on a
relative scale than in Type-1 --- this difference is due to the variations in the 
structure of the covariance matrix. 

\begin{figure}[htp!]
  \centering
\begin{psfrags}
\psfrag{pN1000500}{$n=500/p=1000$}
\psfrag{pN1000500type2}{$n=500/p=1000$}
\psfrag{pN800500}{$n=500/p=800$}
\psfrag{pN800500type2}{$n=500/p=800$}
\psfrag{pN500200}{$n=200/p=500$}
\psfrag{pN500200type2}{$n=200/p=500$}
\psfrag{pN20050}{$n=50/p=200$}
\psfrag{pN20050type2}{$n=50/p=200$}
\includegraphics[width=0.98\textwidth]{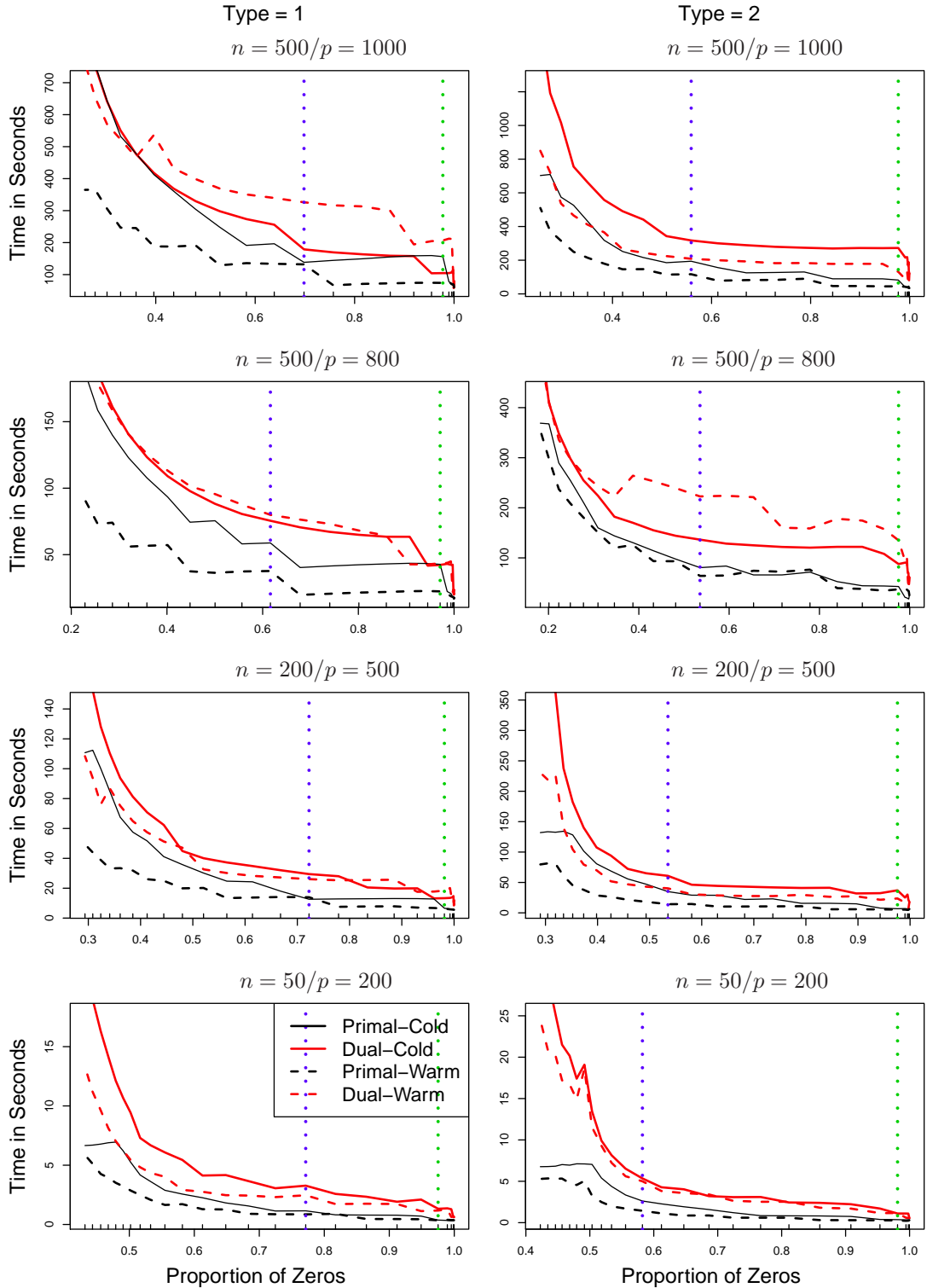}
\end{psfrags}
\caption{\small\em The timings in seconds for the four different
  algorithmic versions: \GL\ (with and without warm-starts) and \DPGL\
  (with and without warm-starts) for a grid of $\lambda$ values
  on the log-scale. [Left Panel] Covariance model for Type-1, [Right Panel] Covariance model for Type-2. 
The horizontal axis is indexed by the proportion
  of zeros in the solution. The vertical dashed lines correspond to
  the optimal $\lambda$ values for which the estimated errors
  $\|\widehat{\B\Theta}_{\lambda} - \B\Theta \|_1$ (green) and
  $\|\widehat{\B\Theta}_{\lambda} - \B\Theta \|_F$ (blue) are minimum.}
\label{fig-times1}
\end{figure}

\begin{table}[htp!!]\label{tab:type1}
\centering
\begin{tabular}{|l| c r r r r c|}
  \hline
 \multirow{2}{*}{ p /  n} & relative  & \multicolumn{4}{c}{Total time (secs) to compute a path of solutions} & 
 Average \%    \\
& error (TOL)    &  Dual-Cold & Dual-Warm &
                         Primal-Cold & Primal-Warm & Zeros in path \\ \hline
 \multirow{2}{*}{1000 / 500} &  $10^{-4}$ & 3550.71 & 6592.63 &
 2558.83 & \bf 2005.25 &  \multirow{2}{*}{80.2} \\ 
        &   $10^{-5}$  & 4706.22 & 8835.59 & 3234.97 &
                            \bf 2832.15 & \\ \hline
 \multirow{2}{*}{1000 / 1000} &    $10^{-4}$ & 2788.30 & 3158.71 &
 2206.95 & \bf 1347.05 &  \multirow{2}{*}{83.0} \\ 
              &   $10^{-5}$  & 3597.21 & 4232.92 & 2710.34 &\bf 1865.57 &
                \\ \hline           
\multirow{2}{*}{1000 / 1500} &    $10^{-4}$ & 2447.19 & 4505.02 &
1813.61 & \bf 932.34 &  \multirow{2}{*}{85.6} \\ 
                &  $10^{-5}$  & 2764.23 & 6426.49 & 2199.53 & \bf 1382.64 &  \\
            \hline\hline
\multirow{2}{*}{800 / 500} &   $10^{-4}$ & 1216.30 & 2284.56 & 928.37
& \bf 541.66 &   \multirow{2}{*}{78.8} \\ 
             &  $10^{-5}$ & 1776.72 & 3010.15 & 1173.76 &\bf 798.93 &  \\
                \hline
\multirow{2}{*}{800 / 800} &     $10^{-4}$  & 1135.73 & 1049.16 &
788.12 & \bf 438.46 &  \multirow{2}{*}{80.0} \\ 
             & $10^{-5}$  & 1481.36 & 1397.25 & 986.19 & \bf  614.98 &  \\  \hline 
\multirow{2}{*}{800 / 1000} &    $10^{-4}$ & 1129.01 & 1146.63 &
786.02 & \bf 453.06 &  \multirow{2}{*}{80.2} \\ 
              &  $10^{-5}$ & 1430.77 & 1618.41 & 992.13 &\bf 642.90 &
\\  \hline \hline
\multirow{2}{*}{500 / 200} &    $10^{-4}$ & 605.45 & 559.14 & 395.11 &
\bf 191.88 &  \multirow{2}{*}{75.9} \\ 
              & $10^{-5}$  & 811.58 & 795.43 & 520.98 &\bf 282.65 &  \\  \hline
\multirow{2}{*}{500 / 500} &    $10^{-4}$  & 427.85 & 241.90 & 252.83 &\bf 123.35&  \multirow{2}{*}{75.2} \\ 
              & $10^{-5}$ & 551.11 & 315.86 & 319.89 &\bf 182.81 & \\  \hline
\multirow{2}{*}{500 / 800} &    $10^{-4}$ & 359.78 & 279.67 & 207.28 &\bf 111.92&  \multirow{2}{*}{80.9} \\ 
              &  $10^{-5}$ & 416.87 & 402.61 & 257.06 &\bf 157.13 &  \\  \hline \hline
\multirow{2}{*}{200 /  50} &    $10^{-4}$ & 65.87 & 50.99 & 37.40 &\bf 23.32 &  \multirow{2}{*}{75.6} \\ 
             & $10^{-5}$ & 92.04 & 75.06 & 45.88 &\bf 35.81 &   \\  \hline
\multirow{2}{*}{200 / 200} &    $10^{-4}$ & 35.29 & 25.70 & 17.32 &\bf 11.72 & \multirow{2}{*}{66.8} \\ 
             &  $10^{-5}$  & 45.90 & 33.23 & 22.41 &\bf 17.16 &  \\  \hline
\multirow{2}{*}{200 / 300} &   $10^{-4}$ & 32.29 & 23.60 & 16.30 &\bf 10.77 &  \multirow{2}{*}{66.0} \\ 
             &  $10^{-5}$ & 38.37 & 33.95 & 20.12 & \bf 15.12 &  \\  \hline
\end{tabular}
\caption{\small\em Table showing the performances of the four algorithms
  \GL\ (Dual-Warm/Cold) and \DPGL\ (Primal-Warm/Cold) for the
  covariance model Type-1. We present the times (in seconds) 
required to  compute a path of
  solutions to~(\ref{eqn-1}) (on a grid of twenty $\lambda$ values) for
  different $(n,p)$ combinations and relative errors
  (as in~(\ref{conv-crit1})). The rightmost column gives the averaged
  sparsity level across the grid of $\lambda$ values. \DPGL\ with
  warm-starts is consistently the winner across all the examples.}\label{tab:type1}
\end{table}

\begin{table}[htp!!]
\centering
\begin{tabular}{|l| c  r r r r c|}
  \hline
\multirow{2}{*}{ p /  n} & relative  & \multicolumn{4}{c}{Total time (secs) to compute a path of solutions} &  Average \%    \\
                         & error (TOL)    &  Dual-Cold & Dual-Warm &
                         Primal-Cold & Primal-Warm & Zeros in path \\ \hline

 \multirow{2}{*}{1000 / 500} & $10^{-4}$ & 6093.11 & 5483.03 & 3495.67
 & \bf 1661.93 & \multirow{2}{*}{75.6} \\ 
                             &   $10^{-5}$ & 7707.24 & 7923.80 &
                             4401.28 &\bf  2358.08 &  \\ \hline

 \multirow{2}{*}{1000 / 1000} &  $10^{-4}$ & 4773.98 & 3582.28 & 2697.38 &\bf 1015.84 & \multirow{2}{*}{76.70} \\ 
                              &   $10^{-5}$ & 6054.21 & 4714.80 &
                              3444.79 & \bf 1593.54 &  \\ \hline

 \multirow{2}{*}{1000 / 1500} &   $10^{-4}$ & 4786.28 & 5175.16 & 2693.39 &\bf 1062.06 & \multirow{2}{*}{78.5} \\ 
                              &    $10^{-5}$ & 6171.01 & 6958.29 & 3432.33 &\bf 1679.16 &  \\  \hline \hline
 
 \multirow{2}{*}{800 / 500} &   $10^{-4}$ & 2914.63 & 3466.49 & 1685.41 &\bf 1293.18 & \multirow{2}{*}{74.3} \\ 
                            &   $10^{-5}$ & 3674.73 & 4572.97 & 2083.20 &\bf 1893.22 &  \\  \hline
 
 \multirow{2}{*}{800 / 800} &   $10^{-4}$ & 2021.55 & 1995.90 & 1131.35 &\bf 618.06 & \multirow{2}{*}{74.4} \\ 
  &   $10^{-5}$ & 2521.06 & 2639.62 & 1415.95 &\bf 922.93 & \\  \hline

\multirow{2}{*}{800 / 1000}  &   $10^{-4}$ & 3674.36 & 2551.06 & 1834.86 &\bf 885.79 & \multirow{2}{*}{75.9} \\ 
                             &   $10^{-5}$ & 4599.59 & 3353.78 &
                             2260.58 &\bf 1353.28 &  \\  \hline \hline

\multirow{2}{*}{500 / 200} &   $10^{-4}$ & 1200.24 & 885.76 & 718.75
&\bf  291.61 & \multirow{2}{*}{70.5} \\ 
                           &   $10^{-5}$ & 1574.62 & 1219.12 & 876.45
                           &\bf 408.41 & \\   \hline
  
\multirow{2}{*}{500 / 500} &   $10^{-4}$ & 575.53 & 386.20 & 323.30 &\bf 130.59 & \multirow{2}{*}{72.2} \\ 
    &   $10^{-5}$ & 730.54 & 535.58 & 421.91 &\bf 193.08 &  \\  \hline

\multirow{2}{*}{500 / 800}  &  $10^{-4}$ & 666.75 & 474.12 & 373.60 &\bf 115.75 & \multirow{2}{*}{73.7} \\ 
                            &  $10^{-5}$ & 852.54 & 659.58 & 485.47 &\bf 185.60 &  \\  \hline \hline

\multirow{2}{*}{200 / 50}  &   $10^{-4}$ & 110.18 & 98.23 & 48.98 &\bf
26.97 & \multirow{2}{*}{73.0} \\ 
                           &   $10^{-5}$ & 142.77 & 133.67 & 55.27
                           &\bf  33.95 &  \\  \hline

\multirow{2}{*}{200 / 200} &   $10^{-4}$ & 50.63 & 40.68 & 23.94 &\bf  9.97 & \multirow{2}{*}{63.7} \\ 
                           &   $10^{-5}$ & 66.63 & 56.71 & 31.57 &\bf 14.70 & \\  \hline

 \multirow{2}{*}{200 / 300} &   $10^{-4}$ & 47.63 & 36.18 & 21.24 &\bf
 8.19 & \multirow{2}{*}{65.0} \\ 
                            &   $10^{-5}$ & 60.98 & 50.52 & 27.41 &\bf 12.22 & \\ 
   \hline
\end{tabular}
\caption{\small\em Table showing comparative timings of the four algorithmic
  variants of \GL\ and \DPGL\ for the covariance model in Type-2. This
table is similar to Table~\ref{tab:type1}, displaying results
for Type-1. \DPGL\ with
  warm-starts consistently outperforms all its competitors.}\label{tab:type2}
\end{table}

\subsection{Micro-array Example}\label{sec:expt-real}  
We consider the data-set introduced in \cite{AB1999} and further
studied in \cite{rothman_2010,MH-GL-11-jmlr}. In this  experiment,
tissue samples were analyzed using an Affymetrix  Oligonucleotide
array. The data was processed, filtered and reduced to a subset of
$2000$ gene expression values. The number of Colon Adenocarcinoma
tissue samples is $n=62$. For the purpose of the experiments presented
in this section, 
we pre-screened the genes to a size of $p=725$. We obtained this subset of genes 
using
the idea of \emph{exact covariance thresholding} introduced in our
paper~\citep{MH-GL-11-jmlr}. We thresholded the sample correlation matrix obtained
from the $62  \times 2000 $ microarray data-matrix 
into connected components with a threshold of $0.00364$\footnote{this is the largest value of
the threshold for which the size of the largest connected component
is smaller than 800} --- the
genes belonging to the largest connected component formed our
pre-screened gene pool of size $p=725$. This (subset) data-matrix of size
$(n,p)=(62,725)$ is used for our experiments. 

The results presented below in Table~\ref{tab:type-micro} show timing
comparisons of the four different algorithms: Primal-Warm/Cold and
Dual-Warm/Cold on a grid of fifteen $\lambda$ values in the
log-scale. Once again we see that the Primal-Warm outperforms the
others in terms of speed and accuracy. Dual-Warm  performs quite
well in this example.

\begin{table}[htp!!]
\centering
\begin{tabular}{|l c c c c|} \hline
 relative  & \multicolumn{4}{c|}{Total time (secs) to compute a path of solutions}   \\
 error (TOL)    &  Dual-Cold & Dual-Warm &  Primal-Cold & Primal-Warm  \\ \hline
$10^{-3}$    & 515.15 & 406.57 & 462.58 &\bf  334.56   \\
 $10^{-4}$  &  976.16 &  677.76 & 709.83 &\bf  521.44 \\ \hline
\end{tabular}
\caption{\small\em Comparisons among algorithms for a microarray dataset with
   $n=62$ and $p=725$, for different tolerance levels (TOL). We took
  a grid of fifteen $\lambda$ values, the average \% of zeros along the
  whole path is $90.8$.}\label{tab:type-micro}
\end{table}

\section{Conclusions}
\label{sec:conclusions}
This paper explores some of the apparent mysteries in the behavior of the \GL\ algorithm introduced in~\cite{FHT2007a}.
These have been explained by leveraging the fact that the \GL\ algorithm is solving the dual of the
graphical lasso problem (\ref{eqn-1}), by block coordinate
ascent. Each block update, itself the solution to a convex program,
is solved via its own dual, which is equivalent to a lasso problem.
The optimization variable is $\M W$, the
covariance matrix, rather than the target precision matrix $\B\Theta$.  During the
course of the iterations, a working version of $\B\Theta$ is
maintained, but it may not be positive definite, and its inverse is
not $\M W$. Tight convergence is therefore essential, for the solution
$\hat{\B\Theta}$ to be a proper inverse covariance.
There are issues using warm starts with \GL, when computing a path of solutions. Unless the sequence of $\lambda$s are sufficiently close, since the  ``warm start''s are not dual feasible, the algorithm can get into trouble.

We have also developed two primal algorithms \PGL\ and \DPGL. The
former is more expensive, since it maintains the relationship
$\M{W}={\B\Theta}^{-1}$ at every step, an $O(p^3)$ operation per sweep
across all row/columns. 
\DPGL\ is similar in flavor to \GL\, except its
optimization variable is $\B\Theta$.  It also solves the dual problem
when computing its block update, in this case a box-QP.  This box-QP
has attractive sparsity properties at {\em both} ends of the
regularization path, as evidenced in some of our experiments.
It maintains a positive definite $\B\Theta$
throughout its iterations, and can be started at any positive definite
matrix. Our experiments show in addition that \DPGL\ is faster than \GL.

An R package implementing \DPGL\ will be made available in CRAN.
\section{Acknowledgements}
We would like to thank Robert Tibshirani and his research
group at Stanford Statistics for helpful discussions. We are also
thankful to the anonymous referees whose comments led to improvements in this
presentation. 
\bibliographystyle{plainnat}
\bibliography{new_agst_new.bib}

\begin{thebibliography}{11}
\providecommand{\natexlab}[1]{#1}
\providecommand{\url}[1]{\texttt{#1}}
\expandafter\ifx\csname urlstyle\endcsname\relax
  \providecommand{\doi}[1]{doi: #1}\else
  \providecommand{\doi}{doi: \begingroup \urlstyle{rm}\Url}\fi

\bibitem[Alon et~al.(1999)Alon, Barkai, Notterman, Gish, Ybarra, Mack, and
  Levine]{AB1999}
U.~Alon, N.~Barkai, D.~A. Notterman, K.~Gish, S.~Ybarra, D.~Mack, and A.~J.
  Levine.
\newblock {Broad patterns of gene expression revealed by clustering analysis of
  tumor and normal colon tissues probed by oligonucleotide arrays}.
\newblock \emph{Proceedings of the National Academy of Sciences of the United
  States of America}, 96\penalty0 (12):\penalty0 6745--6750, June 1999.
\newblock ISSN 0027-8424.
\newblock \doi{10.1073/pnas.96.12.6745}.
\newblock URL \url{http://dx.doi.org/10.1073/pnas.96.12.6745}.

\bibitem[Banerjee et~al.(2008)Banerjee, Ghaoui, and d'Aspremont]{BGA2008}
O.~Banerjee, L.~El Ghaoui, and A.~d'Aspremont.
\newblock Model selection through sparse maximum likelihood estimation for
  multivariate gaussian or binary data.
\newblock \emph{Journal of Machine Learning Research}, 9:\penalty0 485--516,
  2008.

\bibitem[Beck and Teboulle(2009)]{fista-09}
Amir Beck and Marc Teboulle.
\newblock A fast iterative shrinkage-thresholding algorithm for linear inverse
  problems.
\newblock \emph{SIAM J. Imaging Sciences}, 2\penalty0 (1):\penalty0 183--202,
  2009.

\bibitem[Boyd and Vandenberghe(2004)]{BV2004}
Stephen Boyd and Lieven Vandenberghe.
\newblock \emph{Convex Optimization}.
\newblock Cambridge University Press, 2004.

\bibitem[Friedman et~al.(2007)Friedman, Hastie, and Tibshirani]{FHT2007a}
Jerome Friedman, Trevor Hastie, and Robert Tibshirani.
\newblock Sparse inverse covariance estimation with the graphical lasso.
\newblock \emph{Biostatistics}, 9:\penalty0 432--441, 2007.

\bibitem[Hastie et~al.(2009)Hastie, Tibshirani, and Friedman]{FHT-09}
Trevor Hastie, Robert Tibshirani, and Jerome Friedman.
\newblock \emph{The Elements of Statistical Learning, Second Edition: Data
  Mining, Inference, and Prediction (Springer Series in Statistics)}.
\newblock Springer New York, 2 edition, 2009.
\newblock ISBN 0387848576.
\newblock URL
  \url{"http://www.amazon.ca/exec/obidos/redirect?tag=citeulike09-20\&amp;path=ASIN/0387848576"}.

\bibitem[Mazumder and Hastie(2012)]{MH-GL-11-jmlr}
Rahul Mazumder and Trevor Hastie.
\newblock Exact covariance thresholding into connected components for
  large-scale graphical lasso.
\newblock \emph{Journal of Machine Learning Research}, 13:\penalty0 781−794,
  2012.
\newblock URL \url{http://arxiv.org/abs/1108.3829}.

\bibitem[Meinshausen and B\"{u}hlmann(2006)]{MB2006}
N.~Meinshausen and P.~B\"{u}hlmann.
\newblock High-dimensional graphs and variable selection with the lasso.
\newblock \emph{Annals of Statistics}, 34:\penalty0 1436--1462, 2006.

\bibitem[Nesterov(2007)]{nest-07}
Y.~Nesterov.
\newblock Gradient methods for minimizing composite objective function.
\newblock Technical report, Center for Operations Research and Econometrics
  (CORE), Catholic University of Louvain, 2007.
\newblock Tech. Rep, 76.

\bibitem[Rothman et~al.(2008)Rothman, Bickel, Levina, and Zhu]{rothman_2010}
A.J. Rothman, P.J. Bickel, E.~Levina, and J.~Zhu.
\newblock Sparse permutation invariant covariance estimation.
\newblock \emph{Electronic Journal of Statistics}, 2:\penalty0 494--515, 2008.

\bibitem[Yuan and Lin(2007)]{yuan_lin_07}
M~Yuan and Y~Lin.
\newblock Model selection and estimation in the gaussian graphical model.
\newblock \emph{Biometrika}, 94\penalty0 (1):\penalty0 19--35, 2007.

\end{thebibliography}

\newpage

\appendix
\section{Online Appendix}\label{sec:appendix}
This section complements the examples provided in the paper with
further experiments and illustrations. 

\subsection{Examples: Non-Convergence of \GL\ with warm-starts}\label{appendix:example1}
This section illustrates with examples that warm-starts for the \GL\
need not converge. This is a continuation of examples presented in Section~\ref{sec:warm-starts}.\\

\noindent\textbf{Example 1:}\\
We took $(n,p) = (2, 5)$ and setting the seed of the random number
generator in R as {\texttt set.seed(2008) } we generated a
data-matrix $\M{X}_{n \times p}$ with iid standard Gaussian entries. The sample
covariance matrix $\s$ is given below:
\[\begin{pmatrix}  0.03597652 & 0.03792221  & 0.1058585 & -0.08360659  &   0.1366725  \\ 
 0.03597652  & 0.03792221  &  0.1058585 &  -0.08360659 & 0.1366725 \\
0.10585853 & 0.11158361 & 0.3114818& -0.24600689 & 0.4021497 \\
-0.08360659& -0.08812823& -0.2460069&  0.19429514& -0.3176160 \\
0.13667246 & 0.14406402 & 0.4021497& -0.31761603 & 0.5192098 
 \end{pmatrix}\]

With $q$ denoting the maximum off-diagonal entry of $\s$ (in absolute
value), we solved (\ref{eqn-1}) using \GL\ at $\lambda=0.9 \times
q$. The covariance matrix for this $\lambda$ was taken as a warm-start
for the \GL\ algorithm with $\lambda'=\lambda \times 0.01$.  
The smallest eigen-value of the working covariance matrix $\M W$
produced by the \GL\ algorithm, upon updating the first row/column was: 
$\;-0.002896128$, which is clearly undesirable for the convergence of the
algorithm \GL\ . This is why the algorithm \GL\ breaks down.  \\
 
\noindent\textbf{Example 2:}\\
The example is similar to above, with $(n,p)=(10,50)$,  the seed of random number
generator in R being set to {\texttt set.seed(2008) } and $\M{X}_{n\times
  p}$ is the data-matrix with iid Gaussian entries. If the covariance matrix $\widehat{\M{W}_\lambda}$ which solves
problem~(\ref{eqn-1}) with $\lambda = 0.9 \times \max_{i\neq j}
  |s_{ij}|$ is taken as a warm-start to the \GL\ algorithm
with $\lambda'=\lambda \times 0.1$ --- the algorithm fails to
converge. Like the previous example, after the first row/column update, 
the working covariance matrix has negative eigen-values.

\section{Further Experiments and Numerical Studies}\label{sec:times}
This section is a continuation to Section~\ref{rev:sec:timings}, in
that it provides further examples 
comparing the performance of algorithms \GL\ and \DPGL\ .  
The experimental data is generated as follows. For a
fixed value of $p$, we generate a matrix $\M{A}_{p\times p}$ with
random Gaussian entries. The matrix is symmetrized by $\M{A}
\leftarrow (\M{A} + \M {A}')/2 $.  Approximately half of the
off-diagonal entries of the matrix are set to zero, uniformly at
random.  All the eigen-values of the matrix $\M{A}$ are lifted so that
the smallest eigen-value is zero. The noiseless version of the
precision matrix is given by $\B{\Theta}=\M{A} + \tau \M{I}_{p\times
  p}$. We generated the sample covariance matrix $\s$ 
by adding symmetric positive semi-definite random noise $\M N$ to
${\B\Theta}^{-1}$; i.e. $\s={\B \Theta}^{-1}+\M{N}$, where this noise
is generated in the same manner as $\M A$.  We considered four
different values of $p \in \{300,500,800,1000\}$ and two different
values of $\tau \in \{1,4\}$.  

For every $p,\;\tau$ combination we considered a path of twenty $\lambda$ values on the geometric scale. 
For every such case four experiments were performed: Primal-Cold,
Primal-Warm, Dual-Cold and Dual-Warm (as described in Section~\ref{rev:sec:timings}). 
Each combination was run 5 times, and the results averaged, to avoid
dependencies on machine loads.  Figure~\ref{fig-times1} shows the
results. Overall, \DPGL\ with warm starts performs the best,
especially at the extremes of the path. We gave some explanation for this in Section~\ref{sec:solve-cd}.
For the largest problems ($p=1000$) their performances are comparable in the central part of the path (though \DPGL\ dominates), but at the extremes \DPGL\ dominates by a large margin.
\begin{figure}[htpb]
  \centering
\begin{psfrags}
\psfrag{tau = 1 / p=300}{$\tau=1/p=300$}
\includegraphics[width=\textwidth]{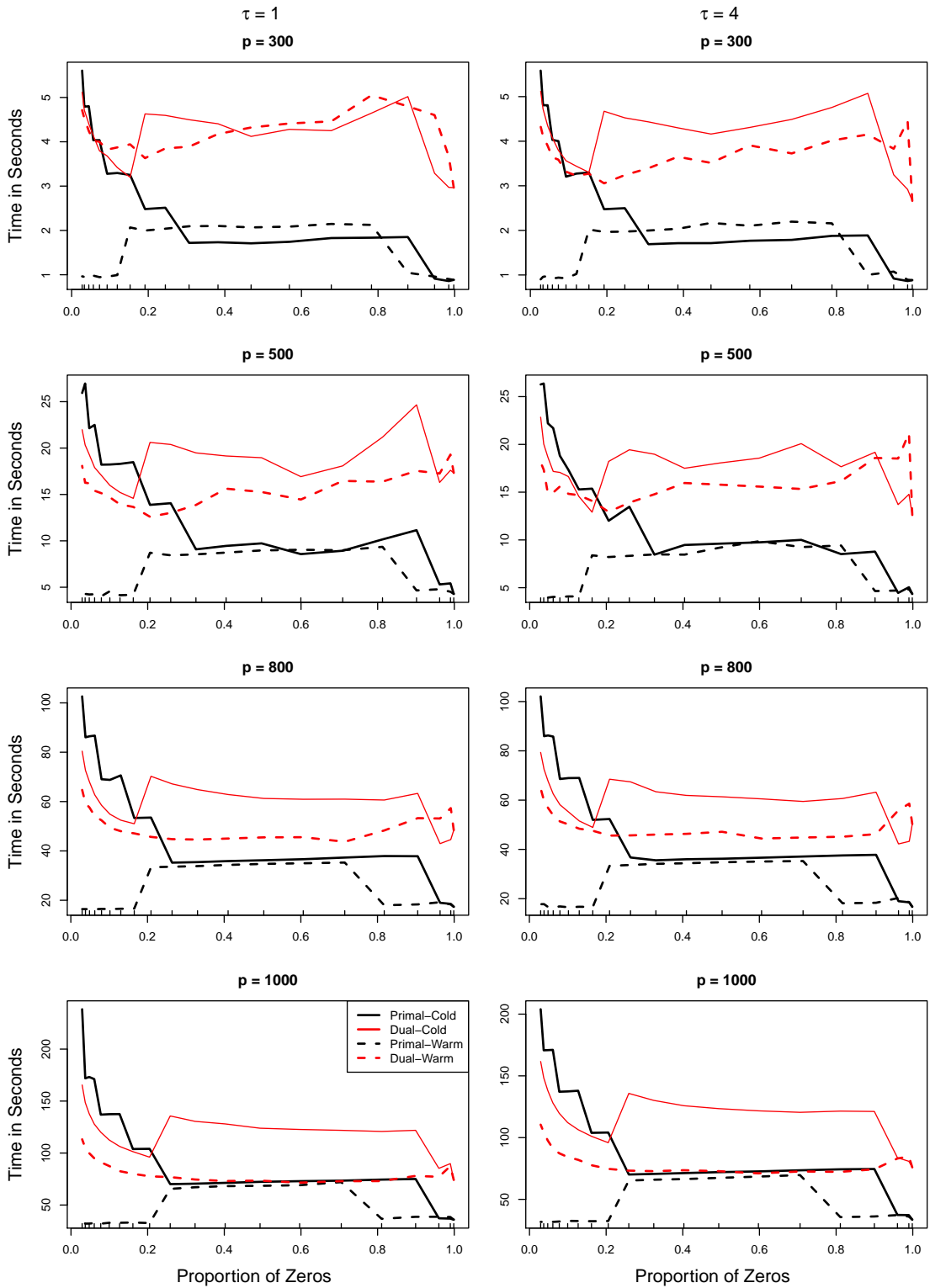}
\end{psfrags}
\caption{The timings in seconds for the four different
  algorithmic versions \GL\ (with and without warm-starts) and \DPGL\
  (with and without warm-starts) for a grid of twenty $\lambda$ values
  on the log-scale. The horizontal axis is indexed by the proportion of zeros in the solution.}
\label{fig-times1}
\end{figure}
\end{document}